\documentclass[letterpaper, 10 pt, conference]{ieeeconf}
\IEEEoverridecommandlockouts

\usepackage[colorlinks]{hyperref}
\hypersetup{
    colorlinks,
    linkcolor=black,
    citecolor=black,
    filecolor=magenta,
    urlcolor=cyan,
}

\usepackage{resizegather}

\usepackage{cite}
\usepackage{amsmath,amssymb,amsfonts,mathrsfs}
\usepackage{algorithmic}
\usepackage{graphicx}
\usepackage{textcomp}
\usepackage{xcolor}
\usepackage{tcolorbox}

\usepackage{tablefootnote}
\usepackage{threeparttable}
\usepackage{caption, subcaption}

\usepackage{tikz}
\usepackage{graphicx}
\usepackage{tkz-euclide}

\usetikzlibrary{positioning}
\usetikzlibrary{shapes}
\usetikzlibrary{shapes.misc}
\usetikzlibrary{shapes.geometric}
\usetikzlibrary{plotmarks}
\usetikzlibrary{intersections}
\usetikzlibrary{calc}
\usetikzlibrary{fit}
\usetikzlibrary{patterns,tikzmark}
\usetikzlibrary{matrix,decorations.pathreplacing,calc}

\tikzset{cross/.style={cross out, draw, 
         minimum size=2*(#1-\pgflinewidth), 
         inner sep=0pt, outer sep=0pt}}

\newcommand{\state}[0]{x}
\newcommand{\prel}[0]{p_{\rm{rel}}}
\newcommand{\vrel}[0]{v_{\rm{rel}}}
\newcommand{\preldot}[0]{\dot{p}_{\rm{rel}}}
\newcommand{\vreldot}[0]{\dot{v}_{\rm{rel}}}

\newcommand{\norm}[1]{\left\lVert#1\right\rVert}

\newcommand{\vertiii}[1]{{\left\vert\kern-0.25ex\left\vert\kern-0.25ex\left\vert #1 
    \right\vert\kern-0.25ex\right\vert\kern-0.25ex\right\vert}}

\newtheorem{theorem}{Theorem}

\newtheorem{definition}{Definition}

\makeatletter
\renewcommand{\fps@figure}{htp}
\renewcommand{\fps@table}{htp}
\makeatother

\def\BibTeX{{\rm B\kern-.05em{\sc i\kern-.025em b}\kern-.08em
    T\kern-.1667em\lower.7ex\hbox{E}\kern-.125emX}}
    
\begin{document}

\title{Safe Legged Locomotion using Collision Cone Control Barrier Functions}

\author{Manan Tayal, Shishir Kolathaya
\thanks{This research was supported by the Prime Minister's Research Fellowship (PMRF).
}
\thanks{Both the authors are from Cyber-Physical Systems, Indian Institute of Science (IISc), Bengaluru.
{\tt\scriptsize \{manantayal, shishirk\}@iisc.ac.in}
.
}%
}

\maketitle
\begin{abstract}
Legged robots exhibit significant potential across diverse applications, including but not limited to hazardous environment search and rescue missions and the exploration of unexplored regions both on Earth and in outer space. However, the successful navigation of these robots in dynamic environments heavily hinges on the implementation of efficient collision avoidance techniques. In this research paper, we employ Collision Cone Control Barrier Functions (C3BF) to ensure the secure movement of legged robots within environments featuring a wide array of static and dynamic obstacles. We introduce the Quadratic Program (QP) formulation of C3BF, referred to as C3BF-QP, which serves as a protective filter layer atop a reference controller to ensure the robots' safety during operation. The effectiveness of this approach is illustrated through simulations conducted on PyBullet.
\end{abstract}


\section{Introduction}
\label{section: Introduction}

\par In recent years, the field of robotics has seen a significant rise in the development and utilization of legged robots. These machines, designed to mimic the movements of animals and humans, have shown great promise in various applications, ranging from search and rescue missions in hazardous environments to exploration of uncharted territories on Earth and beyond. The need for legged robots has become increasingly apparent as they offer distinct advantages over their wheeled or tracked counterparts in terms of adaptability and maneuverability. However, as the complexity and versatility of these robots grow, so does the demand for real-time safe navigation in dynamic and unpredictable environments.


However, the increased complexity of legged robots and their deployment in dynamic, unpredictable environments necessitates the development of real-time safe navigation systems. Ensuring the safety of both the robots and their operators is paramount, as these machines can encounter unforeseen obstacles, difficult terrains, and changing conditions. In this context, real-time safe navigation becomes a critical research area, aiming to enable legged robots to make dynamic decisions and adapt to their surroundings while avoiding collisions, maintaining stability, and achieving their objectives. 

The Collision Cone Control Barrier Functions Based Quadratic Program (C3BF-QP) \cite{C3BF_icc,tayal2023control,tayal2024collision}represents a fusion of Control Barrier Functions \cite{7040372}\cite{Ames_2017} and Collision Cones \cite{Fiorini1993}, \cite{ doi:10.1177/027836499801700706}, \cite{709600}. It operates dynamically in real-time, serving as an additional layer of protection above the reference controller, ensuring the safety of robots. In this research paper, we demonstrate the practical implementation of C3BF-QP on both Quadruped and Bipedal robots, assessing its efficacy in diverse environmental contexts.

\begin{figure}[t]
    \centering
        
        
    \includegraphics[width=0.49\linewidth]{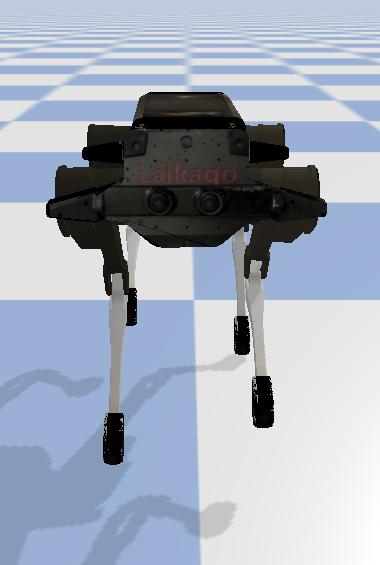}
    \includegraphics[width=0.49\linewidth]{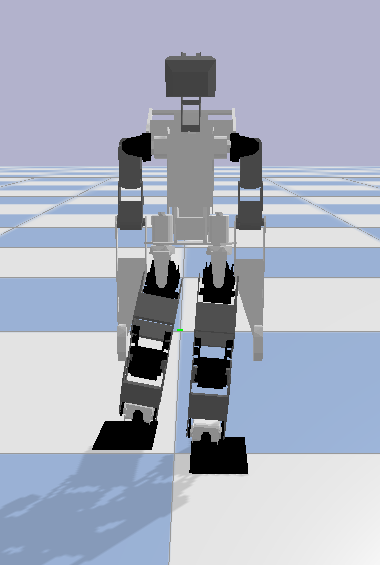}
\caption{Quadruped (left) and Bipedal (right) Legged Robots}
\label{fig:models}
\end{figure}



\subsection{Organisation}
The rest of this paper is organized as follows. Preliminaries explaining the concept of control barrier functions (CBFs), collision cone CBFs and controller design are introduced in section \ref{section: Background}. The application of the above CBFs on the Quadrupeds and Bipeds to avoid obstacles of various shapes is discussed in section \ref{section: Safety Guarantee}. The Simulation setup and results will be discussed in section \ref{section: Simulation Results}. Finally, we present our conclusion in section \ref{section: Conclusions}.

\section{Preliminaries}
\label{section: Background}
In this section, first, we will formally introduce Control Barrier Functions (CBFs) and their importance for real-time safety-critical control. Finally, we will introduce  Collision Cone CBF approach.

\subsection{Control barrier functions (CBFs)}
Having described the vehicle models, we now formally introduce Control Barrier Functions (CBFs) and their applications in the context of safety. 
%
We have the nonlinear control system in affine form:
\begin{equation}
	\dot{\state} = f(\state) + g(\state)u
	\label{eqn: affine control system}
\end{equation}
where $\state \in \mathcal{D} \subseteq \mathbb{R}^n$ is the state of system, and $u \in \mathbb{U} \subseteq \mathbb{R}^m$ the input for the system. Assume that the functions $f: \mathbb{R}^n \rightarrow \mathbb{R}^n$ and $g: \mathbb{R}^n \rightarrow \mathbb{R}^{n \times m}$ are continuously differentiable. Given a Lipschitz continuous control law $u = k(\state)$, the resulting closed loop system $\dot{\state} = f_{cl}(\state) = f(\state) + g(\state)k(\state)$ yields a solution $\state(t)$, with initial condition $\state(0) = \state_0$.
%
Consider a set $\mathcal{C}$ defined as the \textit{super-level set} of a continuously differentiable function $h:\mathcal{D}\subseteq \mathbb{R}^n \rightarrow \mathbb{R}$ yielding,
\begin{align}
\label{eq:setc1}
	\mathcal{C}                        & = \{ \state \in \mathcal{D} \subset \mathbb{R}^n : h(\state) \geq 0\} \\
\label{eq:setc2}
	\partial\mathcal{C}                & = \{ \state \in \mathcal{D} \subset \mathbb{R}^n : h(\state) = 0\}\\
\label{eq:setc3}
	\text{Int}\left(\mathcal{C}\right) & = \{ \state \in \mathcal{D} \subset \mathbb{R}^n : h(\state) > 0\}
\end{align}
It is assumed that $\text{Int}\left(\mathcal{C}\right)$ is non-empty and $\mathcal{C}$ has no isolated points, i.e. $\text{Int}\left(\mathcal{C}\right) \neq \phi$ and $\overline{\text{Int}\left(\mathcal{C}\right)} = \mathcal{C}$. 
The system is safe w.r.t. the control law $u = k(\state)$ if
	$\forall \: \state(0) \in \mathcal{C} \implies \state(t) \in \mathcal{C} \;\;\; \forall t \geq 0$.
We can mathematically verify if the controller $k(\state)$ is safeguarding or not by using Control Barrier Functions (CBFs), which is defined next.

\begin{definition}[Control barrier function (CBF)]{\it
\label{definition: CBF definition}
Given the set $\mathcal{C}$ defined by \eqref{eq:setc1}-\eqref{eq:setc3}, with $\frac{\partial h}{\partial \state}(\state) \neq 0\; \forall \state \in \partial \mathcal{C}$, the function $h$ is called the control barrier function (CBF) defined on the set $\mathcal{D}$, if there exists an extended \textit{class} $\mathcal{K}$ function $\kappa$ such that for all $\state \in \mathcal{D}$:

\begin{equation}
\begin{aligned}
    \underbrace{\text{sup}}_{ u \in \mathbb{U}}\! \left[\underbrace{\mathcal{L}_{f} h(\state) + \mathcal{L}_g h(\state)u} \iffalse+ \frac{\partial h}{\partial t}\fi_{\dot{h}\left(\state, u\right)} \! + \kappa\left(h(\state)\right)\right] \! \geq \! 0
\end{aligned}
\end{equation}
where $\mathcal{L}_{f} h(\state) = \frac{\partial h}{\partial \state}f(\state)$ and $\mathcal{L}_{g} h(\state)= \frac{\partial h}{\partial \state}g(\state)$ are the Lie derivatives.
}
\end{definition}

Given this definition of a CBF, we know from \cite{Ames_2017} and \cite{8796030} that any Lipschitz continuous control law $k(\state)$ satisfying the inequality: $\dot{h} + \kappa( h )\geq 0$ ensures safety of $\mathcal{C}$ if $x(0)\in \mathcal{C}$, and asymptotic convergence to $\mathcal{C}$ if $x(0)$ is outside of $\mathcal{C}$. 

\subsection{Safety Filter Design}
\label{subsection: safe_controller}
Having describe the CBF, we can now describe the Quadratic Programming (QP) formulation of CBFs. CBFs act as \textit{safety filters} which take the desired input $u_{des}(\state,t)$ and modify this input in a minimal way: 

\begin{equation}
\begin{aligned}
\label{eqn: CBF QP}
u^{*}(x,t) &= \min_{u \in \mathbb{U} \subseteq \mathbb{R}^m} \norm{u - u_{des}(x,t)}^2\\
\quad & \textrm{s.t. } \mathcal{L}_f h(x) + \mathcal{L}_g h(x)u + \kappa \left(h(x)\right) \geq 0\\
\end{aligned}
\end{equation}
This is called the Control Barrier Function based Quadratic Program (CBF-QP). The CBF-QP control $u^{*}$ can be obtained by solving the above optimization problem using KKT conditions and is given by
\begin{multline}\label{eq:CBF-QP}
u_{safe}(x, t) \!=\!
	\begin{cases}
		0 & \text{for } \psi(x, t) \geq 0 \\
		-\frac{\mathcal{L}_{g}h(x)^T \psi(x, t)}{\mathcal{L}_{g}h(x)\mathcal{L}_{g}h(x)^T} & \text{for } \psi(x, t) < 0
	\end{cases}
\end{multline}

where $\psi (x,t) := \dot{h}\left(x, u_{ref}(x, t)\right) + \kappa \left(h(x)\right)$. The sign change of $\psi$ yields a switching type of a control law.

\subsection{Collision Cone CBF (C3BF) candidate}
\label{subsection: C3BF}
The concept of a collision cone is a critical element when assessing the potential for collision between two objects. Specifically, it defines a set that enables the prediction of collision likelihood between these objects based on the direction of their relative velocity. In essence, the collision cone for a pair of objects delineates the directions in which, if either object were to move, a collision between them would occur. Throughout the remainder of this paper, we will consider obstacles as ellipses and simplify the ego-vehicle to a single point. As such, the term collision cone will pertain to this scenario, with the center of the ego-vehicle serving as the reference point.

Let's consider a scenario involving an ego-vehicle, described by system (\ref{eqn: affine control system}), and a dynamic obstacle such as a pedestrian or another vehicle. This configuration is visually depicted in Figure 2. To assess potential collisions, we approximate the obstacle as an ellipse and draw two tangents from the center of the ego-vehicle to a conservative circle that encompasses the ellipse, accounting for the dimensions of the ego-vehicle (where $r = max(c_1, c_2) + \frac{Width_{vehicle}}{2}$). For a collision to become a possibility, the relative velocity of the obstacle must be directed toward the ego-vehicle. This implies that the relative velocity vector should not point into the pink shaded region denoted as EHI in Figure \ref{Fig:Construction of Collision Cone}, which takes the shape of a cone. This cone represents a set of unsafe directions for the relative velocity vector, denoted as $\mathcal{C}$. When a function $h:\mathcal{D}\subseteq \mathbb{R}^n \rightarrow \mathbb{R}$ satisfying Definition: \ref{definition: CBF definition} defined on $\mathcal{C}$, it ensures that a Lipschitz continuous control law derived from the resulting QP (\ref{eqn: CBF QP}) for the system guarantees collision avoidance, even if the reference $u_{ref}$ attempts to guide the objects towards a collision course. This innovative approach, avoiding the pink cone region, introduces the concept of Collision Cone Control Barrier Functions (C3BFs).

\begin{figure}
    \centering
    \begin{tikzpicture}[
      collisioncone/.style={shape=rectangle, fill=red, line width=2, opacity=0.30},
      obstacleellipse/.style={shape=rectangle, fill=blue, line width=2, opacity=0.35},
    ]
        
        \def\r{1.32003}; 
        \def\q{-3.5}; 
        \def\x{{\r^2/\q}}; 
        \def\y{{\r*sqrt(1-(\r/\q)^2}}; 
        \def\z{{\q - abs(\q - (\r^2/\q))}};
        \coordinate (Q) at (\q,0); 
        \coordinate (P) at (\x,\y); 
        \coordinate (O) at (0.0, 0); 
        \coordinate (E) at (\q, 0); 
        \coordinate (K) at (\x, {-\y}); 
        \coordinate (H) at (\z, \y);
        \coordinate (I) at (\z, {-\y});
        
        \draw[name path = aux, red!60, very thick, dashed] (O) circle (1.32003);
        \draw[blue!50, thick, fill=blue!20] (O) ellipse (1.20 and 0.55);
        \draw[black, thick] (E) -- (O) node [midway, below] {$\|\prel\|$};
        
        
        \draw[black, thick, name path = tangent] ($(Q)!-0.0!(P)$) -- ($(Q)!1.3!(P)$); 
        \draw[black, thick, name path = normal] ($(O)!-0.0!(P)$) -- ($(O)!1.4!(P)$);
        \draw[black, thick] ($(Q)!-0.0!(K)$) -- ($(Q)!1.3!(K)$);
        
        \draw[black, thick, name path = tangent, dashed] ($(Q)!-0.0!(H)$) -- ($(Q)!1.1!(H)$);
        \draw[black, thick, dashed] ($(Q)!-0.0!(I)$) -- ($(Q)!1.1!(I)$);
        
        \tkzMarkRightAngle[draw=black,size=.2](O,P,Q);
        \tkzMarkAngle[draw=black, size=0.75](O,Q,P);
        \tkzLabelAngle[dist=1.0](O,Q,P){$\phi$};
        
        \path[shade, left color=red, right color = red, opacity=0.2] (E) -- (H) -- (I) -- cycle;
        
        \fill [black] (E) circle (2pt) node[anchor=north, black] (n1) {$(x,y)$} node[anchor=south, black] (n1) {E};
        \fill [blue] (O) circle (2pt) node[anchor=north, blue] (n1) {$(c_x, c_y)$} node[anchor=south east, blue] (n1) {O}; 
        \fill [black] (P) circle (2pt) node[anchor=south, black] (n1) {$P$};
        \fill [black] (K) circle (2pt) node[anchor=north, black] (n1) {$K$};
        \fill [black] (H) circle (2pt) node[anchor=south, black] (n2) {$H$};
        \fill [black] (I) circle (2pt) node[anchor=north, black] (n1) {$I$};
        
        \draw [<->, color=black, thick, dashed] ([xshift=5 pt, yshift=0 pt]O) -- ([xshift=5 pt, yshift=0 pt]P) node [midway, right] {$r = a+\frac{w}{2}$};
        \draw [<->, color=black, thick, dashed] (O) -- (1.20, 0) node [midway, above] {$a$};
        
        \matrix [above right,nodes in empty cells, matrix of nodes, column sep=0.5cm, inner sep=6pt] at (current bounding box.north west) {
          \node [collisioncone,label=right:{\footnotesize Collision Cone}] {}; &
          \node [obstacleellipse,label=right:{\footnotesize Obstacle Ellipse}] {}; \\
        };
    \end{tikzpicture}
    \caption{Construction of collision cone for an elliptical obstacle considering the ego-vehicle's dimensions (width: $w$).} 
    \label{Fig:Construction of Collision Cone}
\end{figure}


\section{Collision Cone CBFs on Legged Robots}
\label{section: Safety Guarantee}
Having described Collision Cone CBF candidate, we will see their application on legged robots in this section.

In this paper, we will focus on two types of legged robots: Quadruped Robots, characterized by their four-legged configuration, and Biped Robots, which possess a two-legged structure. These legged robots are equipped with baseline controllers such as the Convex MPC controller \cite{8594448} for Quadruped robots and the ZMP Walking Using Preview Controller \cite{4209488} \cite{1241826} \cite{maximo2015omnidirectional} for Biped Robots. However, more recent learning based approaches \cite{tayal2022realising,10569443, mothish2024birodiff} can also be used as baseline controllers.
These controllers enable the robots to accurately track velocities or accelerations in the x, y, z and yaw directions. 


Legged robots must navigate while avoiding collisions with both vertical and horizontal obstacles, requiring the utilization of two independent Control Barrier Functions (CBFs) corresponding to each type of obstacle. Furthermore, it is possible to simplify the consideration of all other obstacles by decomposing them into their vertical and horizontal components. In the subsequent subsections, we will elaborate on the derivation of strategies for interacting with these obstacles.

\subsection{Collision Avoidance - Vertical Obstacles}
\label{section: vertical cccbf}
When addressing collision avoidance concerning vertical obstacles, specifically movement within the x-y plane, these robots can be modeled as unicycle model (acceleration control), as demonstrated in \cite{C3BF_icc}. The unicycle model is shown as follows:

\begin{equation}
	\begin{bmatrix}
		\dot{x}_p \\
		\dot{y}_p \\
		\dot{\theta} \\
		\dot{v} \\
		\dot{\omega}
	\end{bmatrix}
	=
        \begin{bmatrix}
            v\cos\theta\\
            v\sin\theta\\
            \omega \\
            0 \\
            0
        \end{bmatrix}
	+
	\begin{bmatrix}
            0 & 0 \\
            0 & 0 \\
            0 & 0 \\
            1 & 0 \\
            0 & 1
	\end{bmatrix}
	\begin{bmatrix}
		a \\
		\alpha
	\end{bmatrix}
    \label{eqn:Acceleration controlled Unicycle model}
\end{equation}

We first obtain the relative position vector between the body center of the robot and the center of the obstacle. 
Therefore, we have
\begin{align}\label{eq:pos_ver}
    \prel := \begin{bmatrix}
        c_x - (x_p + l \cos(\theta)) \\
        c_y - (y_p + l \sin(\theta))
    \end{bmatrix}
\end{align}
Here $l$ is the distance of the body center from the differential drive axis and $\theta$ is the yaw angle. We obtain its velocity as
\begin{align}\label{eq:vel_ver}
    \vrel := \begin{bmatrix}
        \dot c_x - (v \cos (\theta) - l \sin(\theta)*\omega) \\
        \dot c_y - (v \sin (\theta) + l \cos(\theta)*\omega)
    \end{bmatrix}.
\end{align}

From \cite{C3BF_icc}, we have the following CBF candidate:

\begin{equation}
    h_{ver}(\state, t) = <\prel,\vrel>
    + \|\prel\|\|\vrel\|\cos\phi
    \label{eqn:CC-CBF-ver}
\end{equation}

where, $\phi$ is the half angle of the cone, the expression of $\cos\phi$ is given by $\frac{\sqrt{\|\prel\|^2 - r^2}}{\|\prel\|}$ (see Fig. \ref{Fig:Construction of Collision Cone}). 

The constraint simply ensures that the angle between $\prel, \vrel$ is less than $180^\circ - \phi$.  

Having introduced Collision Cone CBF candidates in \ref{subsection: C3BF}, the next step is to formally verify that they are, indeed, valid CBFs. 
We have the following result.

\begin{theorem}\label{thm:CC-CBF}{\it
Given the above model, the proposed CBF candidate \eqref{eqn:CC-CBF-ver} with $\prel,\vrel$ defined by \eqref{eq:pos_ver}, \eqref{eq:vel_ver} is a valid CBF defined for the set $\mathcal{D}$.}
\end{theorem}
\begin{proof}
Similar to the Proof in \cite[Theorem 1]{C3BF_icc}.
\end{proof}

\subsection{Collision Avoidance - Horizontal Obstacles}
\label{section: horizontal cccbf}

When addressing collision avoidance concerning horizontal obstacles, specifically movement within the x-z plane, these robots can be modeled as follows (assuming yaw to be near zero):

\begin{equation}
	\begin{bmatrix}
		\dot{x}_p \\
		\dot{z}_p \\
		\ddot{x}_p \\
		\ddot{z}_p
	\end{bmatrix}
	=
        \begin{bmatrix}
            \dot{x}_p \\
		\dot{z}_p \\
            0 \\
            0
        \end{bmatrix}
	+
	\begin{bmatrix}
            0 & 0 \\
            0 & 0 \\
            1 & 0 \\
            0 & 1
	\end{bmatrix}
	\begin{bmatrix}
		a \\
		{a}_z
	\end{bmatrix}
    \label{eqn:Acceleration controlled model}
\end{equation}

We first obtain the relative position vector between the body center of the robot and the center of the obstacle. 
Therefore, we have
\begin{align}\label{eq:pos_hor}
    \prel := \begin{bmatrix}
        c_x - x_p \\
        c_z - z_p
    \end{bmatrix}
\end{align}
Here $l$ is the distance of the body center from the differential drive axis and $\theta$ is the yaw angle. We obtain its velocity as
\begin{align}\label{eq:vel_hor}
    \vrel := \begin{bmatrix}
        \dot c_x - v \\
        \dot c_z - v_{z}
    \end{bmatrix}.
\end{align}

From \cite{C3BF_icc}, we have the following CBF candidate:

\begin{equation}
    h_{hor}(\state, t) = <\prel,\vrel>
    + \|\prel\|\|\vrel\|\cos\phi
    \label{eqn:CC-CBF-hor}
\end{equation}

where, $\phi$ is the half angle of the cone, the expression of $\cos\phi$ is given by $\frac{\sqrt{\|\prel\|^2 - r^2}}{\|\prel\|}$ (see Fig. \ref{Fig:Construction of Collision Cone}). 

The constraint simply ensures that the angle between $\prel, \vrel$ is less than $180^\circ - \phi$.  

Having introduced Collision Cone CBF candidates in \ref{subsection: C3BF}, the next step is to formally verify that they are, indeed, valid CBFs. 
We have the following result.

\begin{theorem}\label{thm:CC-CBF-hor}{\it
Given the above model, the proposed CBF candidate \eqref{eqn:CC-CBF-hor} with $\prel,\vrel$ defined by \eqref{eq:pos_hor}, \eqref{eq:vel_hor} is a valid CBF defined for the set $\mathcal{D}$.}
\end{theorem}
\begin{proof}
Taking the derivative of \eqref{eqn:CC-CBF-hor} yields
\begin{align}
\dot h_{hor} = &  < \preldot, \vrel > + < \prel, \vreldot >  \nonumber \\
 & + < \vrel, \vreldot > \frac{\sqrt{\|\prel\|^2 - r^2}}{\|\vrel\|} \nonumber \\
 & + < \prel, \preldot > \frac{\|\vrel\| }{\sqrt{\|\prel\|^2 - r^2}}.
 \label{eqn:h_derivative}
\end{align}
Further $\preldot  = \vrel$ and
\begin{align*}
    \vreldot = \begin{bmatrix}
        - a  \\
        -a_{z}
    \end{bmatrix}. \nonumber
\end{align*}
Given $\vreldot$ and $\dot h$, we have the following expression for $\mathcal{L}_g h$:
\begin{align}
    \mathcal{L}_g h = \begin{bmatrix}
        < \prel + \vrel \frac{\sqrt{\|\prel\|^2 - r^2}}{\|\vrel\|}, \begin{bmatrix}
            - 1 \\
            0
        \end{bmatrix}>  \\
                < \prel + \vrel \frac{\sqrt{\|\vrel\|^2 - r^2}}{\|\vrel\|}, \begin{bmatrix}
            0 \\
            - 1
        \end{bmatrix}> 
    \end{bmatrix}^T,
\end{align}
It can be verified that for $\mathcal{L}_gh$ to be zero, we can have the following scenarios:
\begin{itemize}
    \item $\prel + \vrel \frac{\sqrt{\|\prel\|^2 - r^2}}{\|\vrel\|}=0$, which is not possible. Firstly, $\prel=0$ indicates that the vehicle is already inside the obstacle. Secondly, if the above equation were to be true for a non-zero $\prel$, then $\vrel/\|\vrel\| = - \prel/\sqrt{\|\prel\|^2 - r^2}$. This is also not possible as the magnitude of LHS is $1$, while that of RHS is $>1$.
    \item $\prel + \vrel \frac{\sqrt{\|\vrel\|^2 - r^2}}{\|\vrel\|}$ is perpendicular to both $\begin{bmatrix}
            - 1 \\
            0
        \end{bmatrix}$ and  $\begin{bmatrix}
            0 \\
            - 1
        \end{bmatrix}$, which is also not possible.
\end{itemize}
This implies that $\mathcal{L}_gh$ is always a non-zero matrix, implying that $h$ is a valid CBF.
\end{proof}

In the next section, we will look into the simulation experiments on legged robots.

\section{Simulation Results}
\label{section: Simulation Results}
\par We have validated the C3BF-QP based controller on legged robots for both types of obstacles. The simulations were conducted on Pybullet \cite{coumans2019}, a python-based physics simulation engine. PD Controller is used on top of the baseline controllers (Convex MPC and Zero Moment Point(ZMP) Methods mentioned in the previous section) to track the desired path, and the safety controller deployed is given by Section \ref{subsection: safe_controller}. We chose constant target velocities for verifying the C3BF-QP. For the class $\mathcal{K}$ function in the CBF inequality, we chose $\kappa(h) = \gamma h$, where $\gamma=1$.

\begin{figure}
       \centering
        \begin{subfigure}[b]{0.27\textwidth}
        \includegraphics[width=\textwidth]{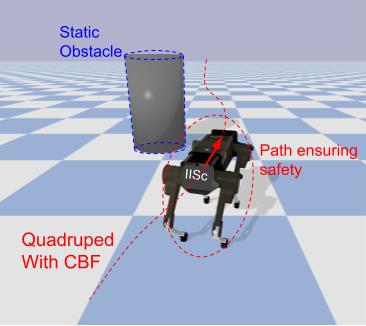}
        \caption{}
        \end{subfigure}
        \begin{subfigure}[b]{0.183\textwidth}
        \includegraphics[width=\textwidth]{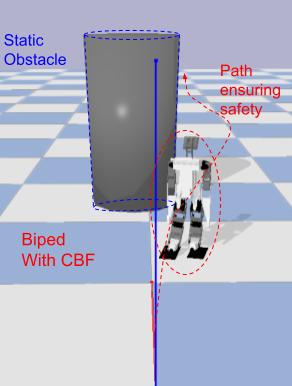}
        \caption{}
        \end{subfigure}
        \caption{Interaction with static obstacles in (a) Quadruped and in (b) Biped Robots}
        \label{fig:static-obs}
    \end{figure}

\begin{figure}
       \centering
        \begin{subfigure}[b]{0.24\textwidth}
        \includegraphics[width=\textwidth]{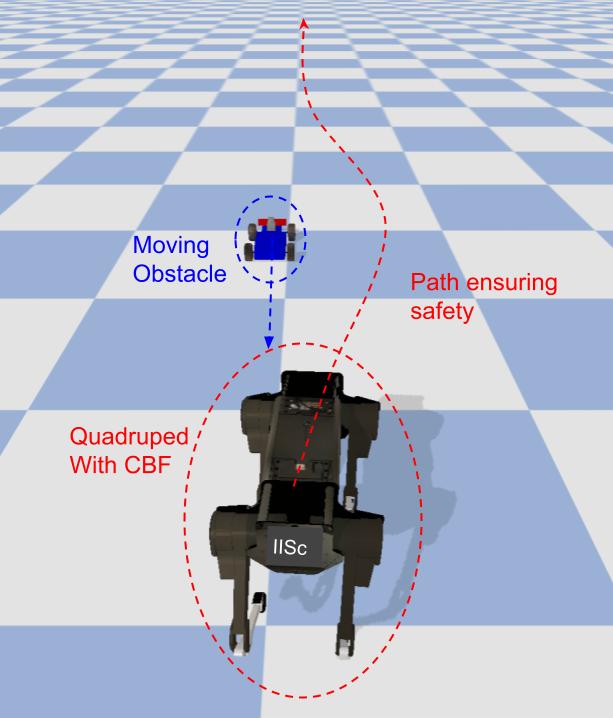}
        \caption{}
        \end{subfigure}
        \begin{subfigure}[b]{0.213\textwidth}
        \includegraphics[width=\textwidth]{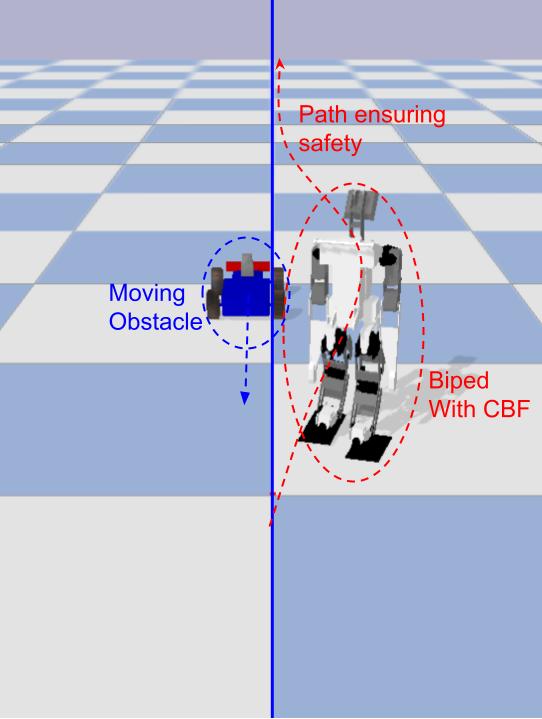}
        \caption{}
        \end{subfigure}
        \caption{Interaction with moving obstacles in (a) Quadruped and in (b) Biped Robots}
        \label{fig:mov-obs}
    \end{figure}

Having presented our proposed control method design, we now test our framework under three different scenarios to illustrate the performance of the controller. These scenarios include the interaction of Quadruped and Biped with: (1) a static obstacle Fig. \ref{fig:static-obs}, (2) a moving obstacle Fig. \ref{fig:mov-obs} and (3) Horizontal obstacle Fig. \ref{fig:hor-obs}.

In case of vertical obstacles (obstacles in x-y plane), the robots avoided the obstacles by moving/ overtaking from the side. In case of horizontal obstacles (obstacles in x-z plane), the robot ducks under the obstacle and regains its height back, after the obstacle is avoided.
The resulting simulations for all the cases can be viewed on the webpage\footnote{\label{note: Pybullet Simulation Video} \url{https://tayalmanan28.github.io/Safe-Legged-C3BF/}}.

\begin{figure}
       \centering
        \begin{subfigure}[b]{0.46\textwidth}
        \includegraphics[width=\textwidth]{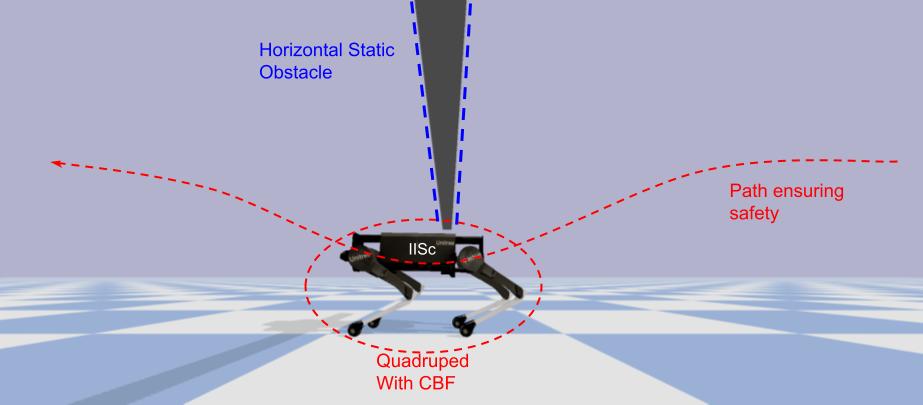}
        \caption{}
        \end{subfigure}
        \caption{Interaction with horizontal obstacle in Quadruped Robot}
        \label{fig:hor-obs}
    \end{figure}

\section{Conclusions}
\label{section: Conclusions}
In conclusion, our research has successfully expanded the application of collision cone Control Barrier Functions (CBFs) to enhance the safety of legged robots when navigating through environments featuring static and dynamic obstacles (as showcased in the accompanying videos on our webpage). Through our work, we have underscored the seamless integration of the proposed C3BF-QP controller with baseline controllers, employed both in quadruped and biped robots. This ease of integration stems from the model-free nature of our formulation, which primarily considers acceleration, highlighting its versatility and potential for real-world deployment.
In our future work, we plan to implement the controller on legged robots in the real-world situations, interacting with a variety of obstacles. We also aim to investigate the potential of safe legged robots navigation in cluttered environments and confined spaces\cite{tayal2023polygonal}. 
Additionally, synthesizing formally verified neural CBFs \cite{tayal2024learning, tayal2024semi} for Legged robots can also be explored.

\newpage
\label{section: References}
\bibliographystyle{IEEEtran}
\bibliography{references.bib}

\begin{thebibliography}{10}
\providecommand{\url}[1]{#1}
\csname url@samestyle\endcsname
\providecommand{\newblock}{\relax}
\providecommand{\bibinfo}[2]{#2}
\providecommand{\BIBentrySTDinterwordspacing}{\spaceskip=0pt\relax}
\providecommand{\BIBentryALTinterwordstretchfactor}{4}
\providecommand{\BIBentryALTinterwordspacing}{\spaceskip=\fontdimen2\font plus
\BIBentryALTinterwordstretchfactor\fontdimen3\font minus
  \fontdimen4\font\relax}
\providecommand{\BIBforeignlanguage}[2]{{%
\expandafter\ifx\csname l@#1\endcsname\relax
\typeout{** WARNING: IEEEtran.bst: No hyphenation pattern has been}%
\typeout{** loaded for the language `#1'. Using the pattern for}%
\typeout{** the default language instead.}%
\else
\language=\csname l@#1\endcsname
\fi
#2}}
\providecommand{\BIBdecl}{\relax}
\BIBdecl

\bibitem{C3BF_icc}
P.~Thontepu, B.~G. Goswami, M.~Tayal, N.~Singh, S.~S. P~I, S.~S. M~G,
  S.~Sundaram, V.~Katewa, and S.~Kolathaya, ``Collision cone control barrier
  functions for kinematic obstacle avoidance in ugvs,'' in \emph{2023 Ninth
  Indian Control Conference (ICC)}, 2023, pp. 293--298.

\bibitem{tayal2023control}
M.~Tayal, R.~Singh, J.~Keshavan, and S.~Kolathaya, ``Control barrier functions
  in dynamic uavs for kinematic obstacle avoidance: A collision cone
  approach,'' in \emph{2024 American Control Conference (ACC)}, 2024, pp.
  3722--3727.

\bibitem{tayal2024collision}
M.~Tayal, B.~G. Goswami, K.~Rajgopal, R.~Singh, T.~Rao, J.~Keshavan, P.~Jagtap,
  and S.~Kolathaya, ``A collision cone approach for control barrier
  functions,'' \emph{arXiv preprint arXiv:2403.07043}, 2024.

\bibitem{7040372}
A.~D. Ames, J.~W. Grizzle, and P.~Tabuada, ``Control barrier function based
  quadratic programs with application to adaptive cruise control,'' in
  \emph{53rd IEEE Conference on Decision and Control}, 2014, pp. 6271--6278.

\bibitem{Ames_2017}
\BIBentryALTinterwordspacing
A.~D. Ames, X.~Xu, J.~W. Grizzle, and P.~Tabuada, ``Control barrier function
  based quadratic programs for safety critical systems,'' \emph{{IEEE}
  Transactions on Automatic Control}, vol.~62, no.~8, pp. 3861--3876, aug 2017.
  [Online]. Available: \url{https://doi.org/10.1109%2Ftac.2016.2638961}
\BIBentrySTDinterwordspacing

\bibitem{Fiorini1993}
P.~Fiorini and Z.~Shiller, ``Motion planning in dynamic environments using the
  relative velocity paradigm,'' in \emph{[1993] Proceedings IEEE International
  Conference on Robotics and Automation}, 1993, pp. 560--565 vol.1.

\bibitem{doi:10.1177/027836499801700706}
\BIBentryALTinterwordspacing
------, ``Motion planning in dynamic environments using velocity obstacles,''
  \emph{The International Journal of Robotics Research}, vol.~17, no.~7, pp.
  760--772, 1998. [Online]. Available:
  \url{https://doi.org/10.1177/027836499801700706}
\BIBentrySTDinterwordspacing

\bibitem{709600}
A.~Chakravarthy and D.~Ghose, ``Obstacle avoidance in a dynamic environment: a
  collision cone approach,'' \emph{IEEE Transactions on Systems, Man, and
  Cybernetics - Part A: Systems and Humans}, vol.~28, no.~5, pp. 562--574,
  1998.

\bibitem{8796030}
A.~D. Ames, S.~Coogan, M.~Egerstedt, G.~Notomista, K.~Sreenath, and P.~Tabuada,
  ``Control barrier functions: Theory and applications,'' in \emph{2019 18th
  European Control Conference (ECC)}, 2019, pp. 3420--3431.

\bibitem{8594448}
J.~Di~Carlo, P.~M. Wensing, B.~Katz, G.~Bledt, and S.~Kim, ``Dynamic locomotion
  in the mit cheetah 3 through convex model-predictive control,'' in \emph{2018
  IEEE/RSJ International Conference on Intelligent Robots and Systems (IROS)},
  2018, pp. 1--9.

\bibitem{4209488}
J.~Park and Y.~Youm, ``General zmp preview control for bipedal walking,'' in
  \emph{Proceedings 2007 IEEE International Conference on Robotics and
  Automation}, 2007, pp. 2682--2687.

\bibitem{1241826}
S.~Kajita, F.~Kanehiro, K.~Kaneko, K.~Fujiwara, K.~Harada, K.~Yokoi, and
  H.~Hirukawa, ``Biped walking pattern generation by using preview control of
  zero-moment point,'' in \emph{2003 IEEE International Conference on Robotics
  and Automation (Cat. No.03CH37422)}, vol.~2, 2003, pp. 1620--1626 vol.2.

\bibitem{maximo2015omnidirectional}
M.~Maximo, ``Omnidirectional zmp-based walking for a humanoid robot,''
  \emph{Master's thesis, Aeronautics Institute of Technology}, 2015.

\bibitem{tayal2022realising}
M.~Tayal and S.~Kolathaya, ``Realising the role of arms on improving the
  stability bipedal robots,'' \emph{preprint}, 2022.

\bibitem{10569443}
G.~Mothish, K.~Rajgopal, R.~Kola, M.~Tayal, and S.~Kolathaya, ``Stoch biro:
  Design and control of a low-cost bipedal robot,'' in \emph{2024 10th
  International Conference on Control, Automation and Robotics (ICCAR)}, 2024,
  pp. 135--140.

\bibitem{mothish2024birodiff}
G.~Mothish, M.~Tayal, and S.~Kolathaya, ``Birodiff: Diffusion policies for
  bipedal robot locomotion on unseen terrains,'' \emph{arXiv preprint
  arXiv:2407.05424}, 2024.

\bibitem{coumans2019}
E.~Coumans and Y.~Bai, ``Pybullet, a python module for physics simulation for
  games, robotics and machine learning,'' \url{http://pybullet.org},
  2016--2019.

\bibitem{tayal2023polygonal}
M.~Tayal and S.~Kolathaya, ``Polygonal cone control barrier functions
  (polyc2bf) for safe navigation in cluttered environments,'' in \emph{2024
  European Control Conference (ECC)}, 2024, pp. 2212--2217.

\bibitem{tayal2024learning}
M.~Tayal, H.~Zhang, P.~Jagtap, A.~Clark, and S.~Kolathaya, ``Learning a
  formally verified control barrier function in stochastic environment,''
  \emph{arXiv preprint arXiv:2403.19332}, 2024.

\bibitem{tayal2024semi}
M.~Tayal, A.~Singh, P.~Jagtap, and S.~Kolathaya, ``Semi-supervised safe
  visuomotor policy synthesis using barrier certificates,'' \emph{arXiv
  preprint arXiv:2409.12616}, 2024.

\end{thebibliography}

\end{document}